\documentclass{article}

\usepackage{natbib}

\usepackage[utf8]{inputenc} %
\usepackage[T1]{fontenc}    %
\usepackage{hyperref}       %
\usepackage{url}            %
\usepackage{booktabs}       %
\usepackage{amsfonts}       %
\usepackage{nicefrac}       %
\usepackage{microtype}      %
\usepackage{MnSymbol}

\usepackage{bm}

\usepackage{notation}
\let\top\intercal
\DeclareMathOperator*{\pr}{\mathbb P}
\DeclareMathOperator*{\ex}{\mathbb E}

\DeclareMathOperator{\KL}{KL}
\DeclareMathOperator{\argmin}{arg\,min}
\let\loss\ell
\DeclareMathOperator{\lossfn}{loss} %
\DeclareMathOperator{\sign}{sign} %
\newcommand{\reals}{\mathbb R}
\newcommand{\nats}{\mathbb N}
\newcommand{\U}{\mathcal U}
\newcommand{\ind}{\mathbf 1}      %
\newcommand{\ip}[2]{\langle #1, #2 \rangle}  %
\newcommand{\normal}{\mathcal{N}}   %
\let\Pr\pr %

\newcommand{\stochleq}{\ensuremath{\leqclosed}}

\DeclareBoldMathCommand{\a}{a}

\usepackage{amsthm}
\newtheorem{theorem}{Theorem}
\newtheorem{lemma}[theorem]{Lemma}

\theoremstyle{definition}
\newtheorem{definition}[theorem]{Definition}
\newtheorem{condition}[theorem]{Condition}
\newtheorem{example}{Example}

\title{Combining Adversarial Guarantees and\\Stochastic Fast Rates in Online Learning}

\author{
  Wouter M. Koolen\\
  Centrum Wiskunde \& Informatica\\
  \texttt{wmkoolen@cwi.nl}
\and
  Peter Gr\"unwald\\
Centrum Wiskunde \& Informatica\\
and Leiden University\\
\texttt{pdg@cwi.nl} 
\and
  Tim van Erven\\
  Leiden University\\
  \texttt{tim@timvanerven.nl}
}

\DeclareRobustCommand{\VAN}[3]{#2} %

\begin{document}

\maketitle

\begin{abstract}
We consider online learning algorithms that guarantee worst-case regret rates in
adversarial environments (so they can be deployed safely and
will perform robustly), yet adapt optimally to favorable stochastic
environments (so they will perform well in a variety of settings of practical importance).
We quantify the friendliness of stochastic
environments by means of the well-known Bernstein (a.k.a.\ generalized
Tsybakov margin)
condition. For two  recent algorithms (Squint for the
Hedge setting and MetaGrad for online convex
optimization) we show that the particular form of their
data-dependent individual-sequence regret guarantees implies that they
adapt automatically to the Bernstein parameters of the stochastic
environment. We prove that these algorithms attain fast rates in their
respective settings both in expectation and with high probability.
\end{abstract}

\section{Introduction}
We consider online sequential decision problems. We focus on full
information settings, encompassing such interaction protocols as online
prediction, classification and regression, prediction with expert advice
or the Hedge setting, and online convex optimization (see
\citealt{CesaBianchiLugosi2006}). The goal of the learner is to choose a
sequence of actions with small regret, i.e.\ such that his cumulative
loss is not much larger than the loss of the best fixed action in
hindsight. This has to hold even in the worst case, where the
environment is controlled by an adversary. After three decades of
research there exist many algorithms and analysis techniques for a
variety of such settings. For many settings, adversarial regret lower
bounds of order $\sqrt{T}$ are known, along with matching individual
sequence algorithms \citep{ShalevShwartz2011}.

A more recent line of development is to design adaptive algorithms with
regret guarantees that scale with some more refined measure of the
complexity of the problem. For the Hedge
setting, results of this type have been obtained, amongst others, by
\citet{CesaBianchiMansourStoltz2007, adahedge,
GaillardStoltzVanErven2014, abprod, llr, squint, adanormalhedge, boa}.
Interestingly, the price for such adaptivity (i.e.\ the
worsening of the worst-case regret bound) is typically extremely small
(i.e.\ a constant factor in the regret bound). For \emph{online convex
optimization} (OCO), many different types of adaptivity have been
explored, including by
\citep{CrammerEtAl2009AROW,adagrad,McMahanStreeter2010,hazan2010extracting,GradualVariationInCosts2012,SteinhardtLiang14,metagrad}.

Here we are interested in the question of whether such adaptive results
are strong enough to lead to improved rates in the stochastic case when
the data follow a ``friendly'' distribution.
In specific cases it has been shown that fancy guarantees do imply
significantly reduced regret. For example
\cite{GaillardStoltzVanErven2014} present a generic argument showing
that a certain kind of second-order regret guarantees implies constant
expected regret (the fastest possible rate) for i.i.d.\ losses drawn
from a distribution with a gap (between expected loss of the best and
all other actions).
In this paper we significantly extend this result. We show that a
variety of individual-sequence second-order regret guarantees imply
fast regret rates for distributions under much milder stochastic
assumptions. In particular, we will look at the Bernstein
condition (see \citealt{bartlett2006empirical}), which is the key to fast rates in
the batch setting.
This condition provides a
parametrised interpolation (expressed in terms of the Bernstein exponent $\kappa \in [0,1]$) between the friendly gap case $(\kappa = 1)$ and the stochastic
worst case $(\kappa = 0$). 
We show that appropriate second-order guarantees
automatically lead to adaptation to these parameters, for both the Hedge
setting and for OCO. In the Hedge setting, we build on the guarantees
available for the Squint algorithm 
\citep{squint} and for OCO we rely on guarantees achieved by
MetaGrad \citep{metagrad} to obtain regret rates of order
$T^{\frac{1-\kappa}{2-\kappa}}$ (Theorem~\ref{thm:main}). We show this, not just in expectation,
but also with high probability.
Our proofs use that, for bounded losses, the
Bernstein condition is equivalent to the so-called 
Central
condition \citep{EGMRW15},
which
provides control over a martingale-type quantity that captures the
second-order part of the bounds (Lemma~\ref{lem:squeezer}). The rates we
obtain include the slow worst-case $\sqrt{T}$ regime for $\kappa = 0$
and the fastest (doubly) logarithmic regime for $\kappa = 1$.

The next section introduces the two settings we consider and the
individual sequence guarantees we will use in each. It also reviews the
stochastic criteria for fast rates and presents our main result. In Section~\ref{sec:examples} we consider a variety of examples illustrating the breadth of cases that we cover. In Section~\ref{sec:results} we prove that second-order guarantees imply adaptation to Bernstein conditions.

\section{Setup}

\subsection{Hedge Setting}\label{sec:hedge}
We start with arguably the simplest setting of online prediction, the Hedge setting popularized by \cite{{FreundSchapire1997}}. To be able to illustrate the full reach of our stochastic assumption we will use a minor extension to countably infinitely many actions $k \in \nats = \set{1,2,\ldots}$, customarily called experts. The protocol is as follows. Each round $t$ the learner plays a probability mass function $w_t = (w_t^1, w_t^2, \ldots)$ on experts. Then the environment reveals the losses $\loss_t = (\loss_t^1, \loss_t^2, \ldots)$ of the experts, where each $\loss_t^k \in [0,1]$. The learner incurs loss $\tuple{w_t, \loss_t} = \sum_k w_t^k \loss_t^k$. The regret after $T$ rounds compared to expert $k$ is given by
\[
R_T^k ~\df~ \sum_{t=1}^T \del*{\tuple{w_t, \loss_t} - \loss_t^k}
.
\]
The goal of the learner is to keep the regret small compared to any expert $k$. We will make use of \emph{Squint} by \citet{squint}, a self-tuning algorithm for playing $w_t$. \citet[Theorem~4]{squint} show that Squint with prior probability mass function $\pi = (\pi^1, \pi^2, \ldots)$ guarantees
\begin{equation}\label{eq:squint.bd}
R_T^k
~\le~
\sqrt{V_T^k K_T^k} + K_T^k
\quad
\text{where}
\quad
K_T^k = O(- \ln \pi^k + \ln \ln T)
\qquad
\text{for any expert $k$.}
\end{equation}
Here $V_T^k \df \sum_{t=1}^T \del*{\tuple{w_t, \loss_t} - \loss_t^k}^2$
is a second-order term that depends on the algorithm's own predictions
$w_t$.
It is well-known that with $K$ experts the worst-case lower bound is
$\Theta(\sqrt{T \ln K})$ \cite[Theorem~3.7]{CesaBianchiLugosi2006}. Taking a fat-tailed prior, for example $\pi^k = \frac{1}{k(k+1)}$, and using $V_T^k \le T$, the above bound implies $R_T^k ~\le~ O \del*{\sqrt{T (\ln k + \ln \ln T)}}$, matching the lower bound in some sense for all $k$ simultaneously.

The question we study in this paper is what becomes of the regret when the sequence of losses $\loss_1, \loss_2, \ldots$ is drawn from some distribution $\pr$, not necessarily i.i.d. But before we expand on such stochastic cases, let us first introduce another setting.

\subsection{Online Convex Optimization (OCO)}\label{sec:oco}
We now turn to our second setting called \emph{online convex
optimization} \citep{ShalevShwartz2011}. Here the set of actions is a compact convex set $\U \subseteq \reals^d$. Each round $t$ the learner plays a point $w_t \in \U$. Then the environment reveals a convex loss function $\loss_t : \U \to \reals$. The loss of the learner is $\loss_t(w_t)$. The regret after $T$ rounds compared to $u \in \U$ is given by
\[
R_T^u ~\df~ \sum_{t=1}^T \del*{\loss_t(w_t) - \loss_t(u)}
.
\]
The goal is small regret compared to any point $u \in \U$. A common tool in the analysis of algorithms is the linear upper bound on the regret obtained from convexity of $\loss_t$ (at non-differentiable points we may take any sub-gradient)
\[
R_T^u
~\le~
\tilde R_T^u
~\df~
\sum_{t=1}^T \tuple{w_t - u, \nabla \loss_t(w_t)}
.
\]
We will make use of (the full matrix version of) \emph{MetaGrad} by \citet{metagrad}. In their
Theorem~8, they show
that, simultaneously, $\tilde R_T^u \le O\del*{D G \sqrt{T}}$ and
\begin{equation}\label{eq:metagrad.bd}
\tilde R_T^u
~\le~
\sqrt{V_T^u K_T} + D G K_T
\quad
\text{where}
\quad
K_T ~=~ O(d \ln T)
\qquad
\text{for any $u \in \U$}
,
\end{equation}
where $D$ bounds the two-norm diameter of $\U$, $G$ bounds $\norm{\nabla \loss_t(w_t)}_2$ the two-norm of the gradients and $V_T^u \df \sum_{t=1}^T \tuple{w_t - u, \nabla \loss_t(w_t)}^2$. The first bound matches the worst-case lower bound. The second bound \eqref{eq:metagrad.bd} may be a factor $\sqrt{K_T}$ worse, as $V_T^u \le G^2 D^2 T$ by Cauchy-Schwarz. Yet in this paper we will show fast rates in certain stochastic settings arising from \eqref{eq:metagrad.bd}. To simplify notation we will assume from now on that $D G = 1$ (this can always be achieved by scaling the loss).

To talk about stochastic settings we will assume that the sequence $\loss_t$ of loss functions (and hence the gradients $\nabla \loss_t(w_t)$) are drawn from a distribution $\pr$, not necessarily i.i.d. This includes the common case of linear regression and classification where $\loss_t(u) = \lossfn(\tuple{u,x_t}, y_t)$ with $(x_t,y_t)$ sampled i.i.d.\ from some distribution and $\lossfn$ a fixed one-dimensional convex loss function (e.g.\ square loss, absolute loss, log loss, hinge loss, \dots).

\subsection{Parameterized Family of Stochastic Assumptions}\label{sec:bern.and.cen}

We now recall the Bernstein \citep{bartlett2006empirical} and Central \citep{EGMRW15} stochastic conditions. In both cases the idea behind the assumption is to control the variance of the excess loss of the actions in the neighborhood of the best action. 

We do not require that the losses are i.i.d., nor that the Bayes act is
in the model. For the Hedge setting it suffices if there is a fixed
expert $k^*$ that is always best, i.e.\ $\ex
\sbrc*{\loss_t^{k^*}}{\mathcal G_{t-1}} = \inf_k \ex
\sbrc*{\loss_t^{k}}{\mathcal G_{t-1}}$ almost surely for all $t$. (Here
we denote by $\mathcal G_{t-1}$ the sigma algebra generated by $\loss_1,
\ldots, \loss_{t-1}$, and the \emph{almost surely} quantification refers to
the distribution of $\loss_1,
\ldots, \loss_{t-1}$.)
Similarly, for OCO we assume there is a fixed point $u^* \in \U$ attaining $\min_{u \in \U} \ex \sbrc*{\loss_t(u)}{\mathcal G_{t-1}}$ at every round $t$. In either case there may be multiple candidate $k^*$ or $u^*$. In the succeeding we assume that one is selected. Note that for i.i.d.\ losses the existence of a minimiser is not such a strong assumption (it is even automatic in the OCO case due to compactness of $\U$), while it is very strong beyond i.i.d. Yet it is not impossible (and actually interesting) as we will show by example in Section~\ref{sec:examples}.

Based on this loss minimiser, we define the \emph{excess losses}, a family of random variables indexed by time $t \in \nats$ and expert/point $k \in \nats$/$u\in \U$ as follows
\begin{equation}\label{eq:excess.loss}
x_t^k ~\df~ \loss_t^k - \loss_t^{k^*} 
\quad \text{(Hedge)}
\qquad
\text{and}
\qquad
x_t^u ~\df~ \tuple{u - u^*, \nabla \loss_t(u)}
\quad \text{(OCO)}
.
\end{equation}
Note that for the Hedge setting we work with the loss directly. For OCO instead we talk about the linear upper bound on the loss, for this is the quantity that needs to be controlled to make use of the MetaGrad bound \eqref{eq:metagrad.bd}. With these variables in place, from this point on the story is the same for Hedge and for OCO. So let us write $\mathcal F$ for either the set $\nats$ of experts or the set $\U$ of points, and $f^*$ for $k^*$ resp. $u^*$, and let us consider the family $\setc{x_t^f}{f \in \mathcal F, t \in \nats}$. We call $f \in \mathcal F$ {\em predictors}. 
 The point of
these stochastic conditions is that they imply that the
variance in the excess loss gets smaller the closer a predictor gets to
the optimum in terms of expected excess loss. This is most directly seen
in the Bernstein condition:
\begin{condition}\label{cond:Bern}
Fix $B \ge 0$ and $\kappa \in [0,1]$. The family \eqref{eq:excess.loss} satisfies the \emph{$(B, \kappa)$-Bernstein condition} if
\[
\ex \sbrc*{(x_t^f)^2}{\mathcal G_{t-1}}
~\le~
B \ex \sbrc*{x_t^f}{\mathcal G_{t-1}}^\kappa
\quad
\text{almost surely for all $f \in \mathcal F$ and rounds $t \in \nats$.}
\]
\end{condition}
Some authors refer to the $\kappa = 1$ case as the \emph{Massart condition}.
As shown by \citet[Theorem~5.4]{EGMRW15}, for bounded excess losses
(which we assume throughout), the Bernstein condition is equivalent to
the following condition:
\begin{condition}\label{con:Cen}
Fix a function $\epsilon : \reals_+ \to \reals_+$. The family \eqref{eq:excess.loss} satisfies the \emph{$\epsilon$-central condition} if
\[
\frac{1}{\eta}
\ln
\ex \sbrc*{e^{- \eta x_t^f}}{\mathcal G_{t-1}}
~\le~
\epsilon(\eta)
\qquad
\text{almost surely for all $f \in \mathcal F$, $\eta \ge 0$ and $t \in \nats$}
.
\]
\end{condition}
\citeauthor{EGMRW15} explicitly convert back and forth between the
parameters of the Bernstein and Central Condition. In the remainder we
will use that Bernstein implies Central (with $\epsilon(\eta) = O( (B
\eta)^{\frac{1}{1-\kappa}})$). For completeness we include a proof in Appendix~\ref{appx:b2c}.

\subsection{Main Result}

In the stochastic case we evaluate the performance of algorithms by
$R_T^{f^*}$, i.e.\ the regret compared to the predictor $f^*$ with
minimal expected loss. The expectation $\ex[R_T^{f^*}]$ is sometimes
called the \emph{pseudo-regret}. The following result shows that
second-order methods automatically adapt to the Bernstein condition.
(Proof sketch in Section~\ref{sec:results}.)

\begin{theorem}\label{thm:main}
In any stochastic setting satisfying the $(B,\kappa$)-Bernstein Condition~\ref{cond:Bern}, the guarantees \eqref{eq:squint.bd} for Squint and \eqref{eq:metagrad.bd} for MetaGrad imply fast rates for the respective algorithms both in expectation and with high probability. That is,
\begin{align*}
\ex[R_T^{f^*}] 
&~=~ 
O\del*{K_T^{\frac{1}{2-\kappa}}T^{\frac{1-\kappa}{2-\kappa}}}
,
\intertext{and for any $\delta > 0$, with probability at least $1-\delta$,}
R_T^{f^*} 
&~=~
O\del*{(K_T - \ln
\delta)^{\frac{1}{2-\kappa}}T^{\frac{1-\kappa}{2-\kappa}}},
\end{align*}
where for Squint $K_T \df K_T^{f^*}$ from~\eqref{eq:squint.bd} and for
MetaGrad $K_T$ is as in~\eqref{eq:metagrad.bd}. 
\end{theorem}
We see that Squint and MetaGrad (and any other second-order methods
achieving the same bounds, as our results only use these
bounds and do not depend on the details of the algorithms)
adapt automatically to the Bernstein parameters of the
distribution, without any tuning. Appendix~\ref{app:cont} provides an
extension of Theorem~\ref{thm:main} that allows using Squint with uncountable ${\cal F}$.

Crucially, the bound provided by Theorem~\ref{thm:main} is natural,
and, in general, the best one can expect. This can be seen from
considering the {\em statistical learning setting}, which is a special
case of our setup. Here $(x_t,y_t)$ are i.i.d.\ $\sim {\mathbb P}$ and
${\cal F}$ is a set of functions from ${\cal X}$ to a set of
predictions ${\cal A}$, with $\ell^f_t := \ell(y_t,f(x_t))$ for some
loss function $\ell: {\cal Y} \times {\cal A} \rightarrow [0,1]$ such
as squared, $0/1$, or absolute loss.  In this setting one usually
considers excess risk, which is the expected loss difference
between the learned $\hat{f}$ and the optimal $f^*$.  The minimax
expected (over training sample $(x^t,y^t)$) risk relative to $f^*$ is
of order ${T}^{-1/2}$ (see e.g.
\cite{massart2006risk,audibert2009fast}).  To get better risk rates,
one has to impose further assumptions on ${\mathbb P}$.  A standard
assumption made in such cases is a Bernstein condition with exponent
$\kappa > 0$; see
e.g. \cite{koltchinskii2006local,bartlett2006empirical}, or
\cite{Audibert04} or \cite{audibert2009fast}; see \cite{EGMRW15} for how it generalizes the Tsybakov margin and other conditions. 

If ${\cal F}$ is sufficiently `simple', e.g. a class with logarithmic
entropy numbers (see Appendix~\ref{app:cont}), or, in classification,
a VC class, then, if a $\kappa$-Bernstein condition holds, ERM
(empirical risk minimization) achieves, in expectation, a better
excess risk bound of order
$O\del*{(\log T) \cdot T^{-\frac{1}{2-\kappa}}}$.  The bound
interpolates between $T^{-1/2}$ for $\kappa = 0$ and $T^{-1}$ for
$\kappa =1$ (Massart condition). Results of
\cite{Tsybakov04,massart2006risk,audibert2009fast} suggest that this
rate can, in general, not be improved upon, and exactly this rate is
achieved by ERM and various other algorithms in various settings by e.g. 
\cite{Tsybakov04,Audibert04,audibert2009fast,bartlett2006convexity}.
By summing from $t=1$ to $T$ and using ERM at each $t$ to classify the
next data point (so that ERM becomes FTL, follow-the-leader), this suggests
that we can achieve a cumulative expected regret $\ex [R^{f^*}_T]$ of
order $O\del*{(\log T) \cdot
  T^{\frac{1-\kappa}{2-\kappa}}}$. Theorem~\ref{thm:main} shows that
this is, indeed, also the rate that Squint attains in such cases if
${\cal F}$ is countable and the optimal $f^*$ has positive prior mass
$\pi^{f^*}$ (more on this condition below)--- we thus see that Squint
obtains exactly the rates one would expect from a statistical
learning/classification perspective, and the minimax excess risk
results in that setting suggests that these cumulative regret rates
cannot be improved in general. It was shown earlier by
\cite{Audibert04} that, when equipped with an oracle to tune the
learning rate $\eta$ as a function of $t$, the rates
$O\del*{(\log T) \cdot T^{\frac{1-\kappa}{2-\kappa}}}$ can also be
achieved by Hedge, but the exact tuning depends on the unknown
$\kappa$. \cite{Grunwald12} provides a means to tune $\eta$
automatically in terms of the data, but his method --- like ERM and
all algorithms in the references above --- may achieve {\em linear\/}
regret in worst-case settings, whereas Squint keeps the $O(\sqrt{T})$
guarantee for such cases.

Theorem~\ref{thm:main} only gives the desired rate for Squint if
${\cal F}$ is countable and $\pi^{f^*} > 0$. The combination of these
two assumptions is strong or at least unnatural, and OCO cannot be
readily used in all such cases either, so in Appendix~\ref{app:cont}
we therefore show how to extend Theorem~\ref{thm:main} to the case of
infinite ${\cal F}$, which can be continuous and thus have
$\pi^{f^*} = 0$, as long as ${\cal F}$ admits sufficiently small
entropy numbers. Incidentally, this also allows us to show that Squint
achieves regret rate
$O\del*{(\log T) \cdot T^{\frac{1-\kappa}{2-\kappa}}}$ when
${\cal F} = \bigcup_{i=1,2,\ldots} {\cal F}_i$ is a countably
infinite union of ${\cal F}_i$ with appropriate entropy numbers; in
such cases there can be, at every sample size, a classifier
$\hat{f} \in {\cal F}$ with $0$ empirical error, so that ERM/FTL will
always overfit and cannot be used even if the Bernstein condition
holds; Squint allows for aggregation of such models. In the remainder of
the main text, we concentrate on applications for which
Theorem~\ref{thm:main} can be used directly, without extensions.

\section{Examples}\label{sec:examples}
We give examples motivating and illustrating the Bernstein/Central condition for the Hedge and OCO settings. Our examples in the Hedge setting will illustrate Bernstein with $\kappa < 1$ and non i.i.d.\ distributions. Our OCO examples were chosen to be natural and illustrate fast rates without curvature.

\subsection{Hedge Setting: Gap implies Bernstein with $\kappa=1$}
In the Hedge setting, we say that a distribution $\pr$ (not necessarily i.i.d.) of expert losses $\setc{\loss_t^k}{t, k \in \nats}$ has \emph{gap} $\alpha > 0$ if there is an expert $k^*$ such that
\[
\ex \sbrc*{\loss_t^{k^*}}{\mathcal G_{t-1}} + \alpha
~\le~
\inf_{k \neq k^*} \ex \sbrc*{\loss_t^k}{\mathcal G_{t-1}}
\qquad
\text{almost surely for each round $ \in \nats$}
.
\]
It is clear that the condition can only hold for $k^*$ the minimiser of the expected loss.
\begin{lemma}
A distribution with gap $\alpha$ is $(\frac{1}{\alpha},1)$-Bernstein.
\end{lemma}
\begin{proof}
For all $k,t$ we have
$
\ex \sbrc*{(x_t^k)^2}{\mathcal G_{t-1}}
\le
1
=
\frac{1}{\alpha} \alpha
\le
\frac{1}{\alpha} \ex \sbrc*{x_t^k}{\mathcal G_{t-1}}
$
.
\end{proof}
By Theorem~\ref{thm:main} we get the $R_T^{k^*} = O(K_T) = O(\ln \ln T)$ rate. \cite{GaillardStoltzVanErven2014} show constant regret for finitely many experts and i.i.d.\ losses with a gap. Our alternative proof above shows that neither finiteness nor i.i.d.\ are essential for fast rates in this case.

\subsection{Hedge Setting: Any $(1,\kappa)$-Bernstein}
The next example illustrates that we can sometimes get the fast rates
without a gap. And it also shows that we can get any intermediate
rate: we construct an example satisfying the Bernstein condition for
any $\kappa \in [0,1]$ of our choosing (such examples occur naturally
in classiciation settings such as those consider in the example in
Appendix~\ref{app:cont}).

Fix $\kappa \in [0,1]$. Each expert $k$ is parametrised by a real
number $\delta_k \in [0,1/2]$. The only assumption we make is that
$\delta_k=0$ for some $k$, and
$\inf_k \setc{\delta_k}{\delta_k>0} = 0$. For a concrete example let
us choose $\delta_1 = 0$ and $\delta_k = 1/k$. Expert $\delta$ has
loss $1/2 \pm \delta$ with probability
$ \frac{1 \pm \delta^{2/\kappa-1}}{2}$ independently between experts
and rounds. Expert $\delta$ has mean loss
$\frac{1}{2} + \delta^{2/\kappa}$, and so $\delta = 0$ is best, with
loss deterministically equal to $1/2$. The squared excess loss of
$\delta$ is $\delta^2$. So we have the Bernstein condition with
exponent $\kappa$ (but no $\kappa' > \kappa$) and constant $1$, and
the associated regret rate by Theorem~\ref{thm:main}.

Note that for $\kappa = 0$ (the hard case) all experts have mean loss equal to $\frac{1}{2}$. So no matter which $k^*$ we designate as the best expert our expected regret is zero. Yet the experts do not agree, as their losses deviate from $\frac{1}{2}$ independently at random. Hence, by the central limit theorem, with high probability our regret is of order $\sqrt{T}$. On the other side of the spectrum, for $\kappa = 1$ (the best case), we do not find a gap. We still have experts arbitrary close to the best expert in mean, but their expected excess loss squared equals their expected excess loss.

ERM/FTL may fail miserably on this type of examples. The clearest case
is when $\setc{k}{\delta_k > \epsilon}$ is infinite for some
$\epsilon > 0$. Then at any $t$ there will be experts that, by chance,
incurred their lower loss every round. Picking any of them will result
in expected instantaneous regret at least $\epsilon^{2/\kappa}$,
leading to linear regret overall.

The requirement $\delta_k = 0$ for some $k$ is essential. If instead $\delta_k > 0$ for all $k$ then there is no best expert in the class. Theorem~\ref{thm:mainb} in Appendix~\ref{app:cont} shows how to deal with this case.

\subsection{Hedge Setting: Markov Chains}

Suppose we model a binary sequence $z_1,z_2,\ldots,z_T$ with $m$-th
order Markov chains. As experts we consider all possible functions $f
\colon \{0,1\}^m \to \{0,1\}$ that map a history of length $m$ to a
prediction for the next outcome, and the loss of expert $f$ is the
$0/1$-loss: $\loss_t^f = |f(z_{t-m},\ldots,z_{t-1}) - z_t|$. (We
initialize $z_{1-m} = \ldots = z_{0} = 0$.) A uniform prior on this finite set of $2^{2^m}$ experts results in worst-case regret of order $\sqrt{T 2^m}$.
Then, if the data are
actually generated by an $m$-th order Markov chain with transition
probabilities $\Pr(z_t = 1 \mid (z_{t-m},\ldots,z_{t-1}) = \a) =
p_{\a}$, we have $f^*(\a) = \ind\{p_\a \geq \half\}$ and
\begin{align*}
\ex \sbrc*{(x_t^f)^2}{(z_{t-m},\ldots,z_{t-1}) = \a}
  &= 1,
  &
  \ex \sbrc*{x_t^f}{(z_{t-m},\ldots,z_{t-1}) = \a}
  &= 2|p_\a - \half|
\end{align*}
for any $f$ such that $f(\a) \neq f^*(\a)$. So the Bernstein
condition holds with $\kappa = 1$ and
$B = \frac{1}{2 \min_\a |p_\a - \half|}$.

\subsection{OCO: Hinge Loss on the Unit Ball}

Let $(x_1, y_1), (x_2,y_2),\ldots$ be classification data, with $y_t \in
\{-1,+1\}$, and consider the \emph{hinge loss} $\loss_t(u) = \max
\set*{0, 1- y_t \tuple{x_t, u}}$. Now suppose, for simplicity, that both
$x_t$ and $u$ come from the $d$-dimensional unit Euclidean ball, such
that $\ip{x_t}{u} \in [-1,+1]$ and the hinge is never active, i.e.\
$\loss_t(u) = 1 - y_t \tuple{x_t, u}$. Then, if the data turn out to be
i.i.d.\ observations from a fixed distribution $\pr$, the Bernstein
condition holds with $\kappa = 1$ (The proof can be found in Appendix~\ref{sec:pf5}):

\DeclareBoldMathCommand{\X}{X}
\DeclareBoldMathCommand{\Z}{Z}
\DeclareBoldMathCommand{\u}{u}
\DeclareBoldMathCommand{\vmu}{\mu}
\DeclareBoldMathCommand{\w}{w}
\DeclareBoldMathCommand{\e}{e}
\DeclareBoldMathCommand{\zeros}{0}
\newcommand{\lambdamax}{\lambda_\text{max}}

\begin{lemma}[Unregularized Hinge Loss Example]\label{lem:hingeloss}
  Consider the hinge loss setting above, where $|\ip{x_t}{u}| \leq 1$. If
  the data are i.i.d., then the $(B,\kappa)$-Bernstein condition is
  satisfied with $\kappa = 1$ and $B = \frac{2\lambdamax}{\|\vmu\|}$,
  where $\lambdamax$ is the maximum eigenvalue of $\ex\sbr*{\X \X^\top}$
  and $\vmu = \ex\sbr{Y\X}$, provided that $\|\vmu\| > 0$.

  In particular, if $\X_t$ is uniformly distributed on the sphere and
  $Y_t = \sign(\ip{\bar{\u}}{\X_t})$ is the noiseless classification of
  $\X_t$ according to the hyperplane with normal vector $\bar{\u}$, then $B
  \leq \frac{c}{\sqrt{d}}$ for some absolute constant $c > 0$.
\end{lemma}

The excluded case $\|\vmu\| = 0$ only happens in the degenerate case
that there is nothing to learn, because $\vmu = \zeros$ implies that the
expected hinge loss is $1$, its maximal value, for all $\u$.

\subsection{OCO: Absolute Loss}
Let $\U = [0,1]$ be the unit interval. Consider $\loss_t(u) =
\abs{u-x_t}$ where $x_t \in [0,1]$ are drawn i.i.d.\ from $\pr$. Let
$u^* \in \arg\min_u \ex \abs{u-x}$ minimize the expected loss. In this
case we may simplify $\tuple{w-u^*, \nabla \loss(w)} =
(w-u^*)\sign(w-x)$. To satisfy the Bernstein condition, we therefore
want $B$ such that, for all $w \in [0,1]$,
\[
\ex \sbr*{\del[\big]{(w - u^*) \sign \del{w-x}}^2}
~\le~
B
\ex \sbr*{(w - u^*) \sign \del{w-x}}^\kappa.
\]
That is,
\[
|w - u^*|^{2-\kappa}
~\le~
B 2^\kappa |\Pr(x \leq w) - \half|^\kappa.
\]
For instance, if the distribution of $x$ has a strictly positive density
$p(x) \geq m > 0$, then $u^*$ is the median and $|\Pr(x \leq w) - \half|
= |\Pr(x \leq w) - \Pr(x \leq u^*)| \geq m |w-u^*|$, so the condition
holds with $\kappa = 1$ and $B = \frac{1}{2m}$. Alternatively, for a
discrete distribution on two points $a$ and $b$ with probabilities $p$
and $1-p$, the condition holds with $\kappa = 1$ and $B =
\frac{1}{|2p-1|}$, provided that $p \neq \half$, as can be seen by
bounding $|w-u^*| \leq 1$ and $|\Pr(x \leq w) - \half| \geq |p-\half|$.

\section{Proof of Main Result}\label{sec:results}
This section builds up to prove our main result Theorem~\ref{thm:main}. 
We first introduce a handy abbreviation that allows us to reason simultaneously in expectation and with high probability. We then identify the minimal $\epsilon$ for which the Central Condition~\ref{con:Cen} holds. We then show how we can introduce a second-order adjustment, and characterize the cost. Combination with either worst-case regret bound then yields the desired result.

\subsection{Notation: Exponential Stochastic Negativity and Inequality}\label{sec:ESI.notation}
We introduce a convenient shorthand notation that we will use throughout this paper.

\begin{definition}\label{def:ESI}
A random variable $X$ is \emph{exponentially stochastically negative}, denoted $X \stochleq 0$, if $\ex \sbr{e^X} \le 1$. For any $\eta \ge 0$, we write $X \stochleq_\eta 0$ if $\eta X \stochleq 0$. For any pair of random variables $X$ and $Y$, we say that $X$ is \emph{exponentially stochastically less} than $Y$, denoted $X \stochleq Y$, if $X-Y \stochleq 0$.
\end{definition}

\begin{lemma}
\label{lem:ESI}\label{lem:mix}\label{lem:chainrule}
Exponential stochastic negativity has the following useful properties:
\begin{enumerate}
\item (Negativity). Let $X \stochleq 0$. As the notation suggests $X$ is negative in expectation and with high probability. That is $\ex\sbr*{X} \le 0$ and $\pr \set*{X \ge - \ln \delta} \le \delta$ for all $\delta > 0$.
\item (Convex combination). Let $\set*{X^f}_{f \in \mathcal F}$ be a family of random variables and let $w$ be a distribution on $\mathcal F$. If $X^f \stochleq 0$ for all $f$ then $\ex_{f \sim w} [X^f] \stochleq 0$.
\item (Chain rule). Let $X_1, X_2, \ldots$ be adapted to filtration $\mathcal G_1 \subseteq \mathcal G_2 \ldots$ (i.e.\ $X_t$ is $\mathcal G_t$-measurable for each $t$). If $X_t | \mathcal G_{t-1} \stochleq 0$  almost surely for all $t$, then $\sum_{t=1}^T X_t \stochleq 0$ for all $T \ge 0$.
\end{enumerate}
\end{lemma}

\begin{proof}
\emph{Negativity}: 
By Jensen's inequality $\ex\sbr*{X} \le \ln \ex\sbr*{ e^X} \le 0$, whereas by Markov's inequality $\pr \set*{X \ge - \ln \delta} = \pr \set*{e^X \ge \frac{1}{\delta}} \le \delta \ex \sbr*{e^X} \le \delta$.
\emph{Convex combination}:
By Jensen's inequality
$
\ex \sbr*{e^{\ex_{f \sim w}[X^f]}}
\le
\ex_{f \sim w} \ex \sbr*{e^{X^f}}
\le
1
$.
\emph{Chain rule}:
By induction. The base case $T=0$ holds trivially. For $T > 0$ we have
$
\ex \sbr*{
e^{\sum_{t=1}^T X_t}
}
=
\ex \sbr*{
  e^{\sum_{t=1}^{T-1} X_t}
  \ex \sbrc*{
    e^{X_T}
  }{\mathcal G_{T-1}}
}
\le
\ex \sbr*{
  e^{\sum_{t=1}^{T-1} X_t}
}
\le
1
.
$
\end{proof}

\subsection{Normalized Cumulant Generating Function}

To prove our main result we will make use of the Central Condition~\ref{con:Cen}. For any distribution $\pr$ this condition will hold for some $\epsilon$ (which may be trivial). In this section we construct the smallest $\epsilon$ for which it holds and derive a few useful properties of that $\epsilon$.

\bigskip\noindent
Consider the family \eqref{eq:excess.loss} of excess loss variables $x_t^f$. We assume that $x_t^f \in [-1,1]$  is bounded in a range of width $2$ and has positive mean $\ex \sbrc{x_t^f}{\mathcal G_{t-1}} \ge 0$ by definition of $f^*$. As we will see, the complexity of our learning problem will be governed by the distribution of $x_t^f$. In particular, we will look at the \emph{normalized cumulant generating function} for $\eta \ge 0$:
\[
\epsilon_t^f(\eta) 
~\df~ 
\frac{1}{\eta} \ln \ex \sbrc*{e^{-\eta x_t^f}}{\mathcal G_{t-1}}
\]
By construction $-x_t^f \stochleq_\eta \epsilon_t^f(\eta)$. Boundedness of $x_t^f \in [-1,1]$ immediately results in $\epsilon_t^f(\eta) \in [-1,1]$. Moreover, Hoeffding's inequality (see e.g.\ \citet[Lemma~2.2]{CesaBianchiLugosi2006}) tells us that $\epsilon_t^f(\eta) \le \eta/2$
while Jensen's inequality gives $\epsilon_t^f(\eta) \ge - \ex \sbrc*{x_t^f}{\mathcal G_{t-1}}$. The dual representation $\epsilon_t^f(\eta) = \sup_Q - \ex_Q \sbr*{x} - \frac{1}{\eta} \KL\delcc[\big]{Q(x)}{\pr\delc{x_t^f}{\mathcal G_{t-1}}}$ reveals that $\epsilon_t^f(\eta)$ is increasing in $\eta$. The value at $\eta=0$ is obtained by continuity from $\epsilon_t^f(0) \df \lim_{\eta \to 0} \epsilon_t^f(\eta) = - \ex \sbrc{x_t^f}{\mathcal G_{t-1}} \le 0$.

To get a uniform control over the class $\mathcal F$, we will make use of the maximum 
\begin{equation}\label{eq:uniform.eta}
\epsilon_t(\eta)
~\df~
\sup_{f \in \mathcal F}~
\epsilon_t^f(\eta)
\qquad
\text{and}
\qquad
\epsilon(\eta)
~\df~
\sup_{t}~
\epsilon_t(\eta)
.
\end{equation}
The functions $\epsilon_t$ and $\epsilon$ inherit most properties of each $\epsilon_t^f$, but in addition since $f^* \in \mathcal F$ and $\epsilon_t^{f^*}(\eta) = 0$, we see that $\epsilon_t(\eta) \ge 0$ and also that $\epsilon_t(0) = 0$. Moreover, since $\epsilon_t(\eta) \le \eta/2$ we have $\lim_{\eta \to 0} \epsilon_t(\eta) = 0$. In this paper we will judge the complexity of the interplay of the generating distribution $\pr$ with the class $\mathcal F$ by how $\epsilon_t(\eta) \to 0$ as $\eta \to 0$. By construction the Central Condition~\ref{con:Cen} holds with $\epsilon(\eta)$.

\subsection{A Second-order Adjustment to Exponential Stochastic Inequality}

We now show a technical lemma showing that, roughly, the square of a bounded Central random variable is exponentially stochastically less than that variable itself. Consider any random variable $x \in [-1,1]$, and let us denote its normalized cumulant generating function by
$
\epsilon(\eta) = \frac{1}{\eta} \ln \ex \sbr*{e^{-\eta x}}
$.
(In particular, see Definition~\ref{def:ESI}, $- x \stochleq_\eta \epsilon(\eta)$ for all $\eta \ge 0$.)
Intuitively,  small $\epsilon(\eta) \ll \eta/2$ is special, indicating that $x$ cannot be often very negative. The following lemma shows that if $x$ is special to degree $\epsilon(\eta)$, then the smaller quantity $\approx x - \frac{\eta}{2} x^2$ is also special at only mildly weaker degree $\approx \epsilon(2 \eta)$.

\begin{lemma}\label{lem:squeezer}
For any random variable $x \in [-1,1]$ and any $\eta \ge 0$
\[
\frac{1}{\eta} \ln \ex \sbr*{e^{c \eta^2 x^2 - \eta x}} 
~\le~ 
\epsilon(2 \eta) + c \eta \epsilon(2 \eta)^2
\qquad
\text{where}
\qquad
c
~=~
\frac{1}{1 + \sqrt{1 + 4 \eta^2}}
.
\]
\end{lemma}
In the notation of Section~\ref{sec:ESI.notation}, the lemma reads
\[
- x ~\stochleq_\eta~ \epsilon(\eta)
\quad
\text{for all $\eta \ge 0$}
\quad
\text{implies}
\quad
c \eta x^2 -  x
~\stochleq_\eta~
\epsilon(2 \eta) + c \eta \epsilon(2 \eta)^2
\quad
\text{for all $\eta \ge 0$}
.
\]

\begin{proof}
By Theorem~\ref{thm:variance} in Appendix~\ref{sec:esi.soa} with $\gamma = 2 \eta$ and the largest admissible $c$ for \eqref{eq:good.c}.
\end{proof}

Note that for $\eta = 0$ the lemma trivializes, telling us $- \ex \sbr*{x} \le \epsilon(0)$ where we have in fact equality. Note also that the right-hand side is an increasing function in $\epsilon(2 \eta)$ (the quadratic in $\epsilon(2 \eta)$ has positive derivative for all $\epsilon(2 \eta) \ge -\frac{1+\sqrt{4 \eta^2+1}}{2 \eta} < - 1$.)

\subsection{From Second-order Bound to Bound in Terms of Parameter of Distribution}
The next step toward fast rates is to obtain from a second-order bound, which involves the algorithm, another bound strictly in terms of the parameters of the distribution, which do not refer to the algorithm. The proof is in Appendix~\ref{sec:pf9}.

\begin{theorem}\label{thm:luckiness.regret.bd}
Consider either Squint in the Hedge setting or MetaGrad for OCO. Let  $\setc*{x_t^f}{f \in \mathcal F}$ be the associated the excess loss family from \eqref{eq:excess.loss}, and let $\epsilon(\eta)$ be, as in \eqref{eq:uniform.eta}, the corresponding maximal normalized cumulant generating function. For the Hedge setting let $K_T \df K_T^{f^*}$ as in \eqref{eq:squint.bd}, for OCO let $K_T$ be as in \eqref{eq:metagrad.bd}.
Then for each $\gamma \ge 0$ with $c$ as in Lemma~\ref{lem:squeezer},
\[
R_T^{f^*}
~\stochleq_\gamma~
\frac{K_T}{c \gamma}
+ T \epsilon(2 \gamma) (1 + c \gamma^2)
+ 2 K_T
.
\]
\end{theorem}

To prove our main Theorem~\ref{thm:main} we invoke the Bernstein Condition~\ref{cond:Bern} to bound $\epsilon(2\epsilon)$ as a polynomial in $\gamma$, and then tune $\gamma$ to optimize the bound. The details of the proof can be found in Appendix~\ref{sec:pf3}.

\section{Conclusion}
We show that it is possible for online learning methods to provide both the safety and robustness of a worst-case regret bound and be adaptive to favorable stochastic environments. We focus on Squint and MetaGrad, methods for online learning with individual sequence regret guarantees of a particular second order form. We show that such guarantees imply automatic adaptivity to the Bernstein parameters of stochastic environments, and result in the corresponding fast regret rates.

\DeclareRobustCommand{\VAN}[3]{#3} %

\bibliography{../colt/bib}

\DeclareRobustCommand{\VAN}[3]{#2} %

\cleardoublepage
\appendix

\section{Second-order Adjustment of Exponential Stochastic Inequality}\label{sec:esi.soa}

In this section we prove a stronger form of Lemma~\ref{lem:squeezer}. We would like to remark that our solution to this problem was inspired by the general moments problem studied by \citet[Section 3]{NIPS2014_5434}, especially because this connection became invisible during the simplification of our proofs.

We will be thinking about two learning rates, $0 \le \eta \le \gamma$. The larger one, $\gamma$, will be where we evaluate $\epsilon(\gamma)$. So $\gamma$ controls the strength of the assumption. The smaller one, $\eta$, will be the learning rate at which we obtain the conclusion. The point is to get a large amount of quadratic $x^2$ in the conclusion, as governed by the constant $c$. Obviously, the more greedy we are in $\eta$ and $\gamma$, the smaller the $c$ for which we can get any traction. This trade-off is captured by the following relationship between $\gamma$, $\eta$ and $c$ that we will make use of throughout this section.

\begin{equation}\label{eq:good.c}
0
~\le~
c
~\le~
\frac{\sqrt{2 \abs{2 \eta - \gamma} +\gamma^2+1}-\abs{2 \eta - \gamma} -1}{4 \eta^2}
\end{equation}
Positivity of $c$ is not that important, as the desired inequality is trivial for $c \le 0$. The following inequality is useful later.

\begin{lemma}\label{lem:derv.sign}
Let $0 \le \eta \le \gamma$ and $c$ satisfy \eqref{eq:good.c}. Then
\[
1  ~\ge~ 2 c \eta
\]
\end{lemma}

\begin{proof}
We need to show
\[
2 \eta
~\ge~
\sqrt{2 \abs{2 \eta - \gamma} +\gamma^2+1}-\abs{2 \eta - \gamma} -1
\]
that is
\[
\del*{
  2 \eta
  + \abs{2 \eta - \gamma}
  + 1
}^2
~\ge~
2 \abs{2 \eta - \gamma} +\gamma^2+1
\]
Expanding the left-hand side square results in
\[
4 \eta^2 + 4 \eta \abs{2 \eta - \gamma} + 4 \eta
+ \abs{2 \eta - \gamma}^2 + 2 \abs{2 \eta - \gamma} + 1
~=~
4 \eta (2\eta - \gamma)
+ 4 \eta \abs{2 \eta - \gamma}
+ 4 \eta
+ \gamma^2
+ 2 \abs{2 \eta - \gamma}
+ 1
\]
which definitely exceeds the right-hand side above.
\end{proof}

We now put our assumption to use. In the following Lemma we show that it implies a not-in-expectation-but-with-a-correction-term version of the result we are after.

\begin{lemma}\label{lem:pointwise}
Fix $0 \le \eta \le \gamma$ and let $c$ satisfy \eqref{eq:good.c}. Then for each $x \in [-1,1]$ and $\epsilon \in [-1,1]$ we have
\[
e^{c \eta^2 x^2 - \eta x}
-
\frac{  e^{- \gamma (x + \epsilon)}-1}{\gamma} \eta (1 + 2 c \eta \epsilon) e^{c \eta^2 \epsilon^2 + \eta \epsilon}
~\le~
e^{c \eta^2 \epsilon^2 + \eta \epsilon}
.
\]
\end{lemma}

\begin{proof}
We will show that the left-hand side is maximized over $x \in [-1,1]$ at $x=-\epsilon$. First, its derivative equals
\[
e^{- \gamma x}
\eta
\del[\big]{
  h(-\epsilon)
  -
  h(x)
}
\qquad
\text{where}
\qquad
h(x) = (1 - 2 c \eta x) e^{c \eta^2 x^2 + (\gamma - \eta) x}
.
\]
This indeed equals zero at $x=-\epsilon$. To show that $x=-\epsilon$ is indeed a maximum and that there are no other maxima it suffices to show that $h(x)$ is increasing on $x \in [-1,1]$. We have
\[
h'(x)
~=~
\del*{
-4 c^2 \eta^3 x^2+2 c \eta x (2 \eta-\gamma)- 2 c \eta  +\gamma- \eta
} e^{c \eta^2 x^2 + (\gamma - \eta) x}
\]
As the term in parentheses is concave in $x$, it suffices to show that $h'(x) \ge 0$ for $x \in \set{-1,1}$, i.e.\
\[
-4 c^2 \eta^3 - 2 c \eta \abs{2 \eta-\gamma}- 2 c \eta  +\gamma- \eta
~\ge~
0
\]
Solving the quadratic in $c$, we see that this holds if \eqref{eq:good.c}, as required.
\end{proof}

Finally, we are ready for the general version of the claim.

\begin{theorem}\label{thm:variance}
Pick $0 \le \eta \le \gamma$ and $c$ satisfying \eqref{eq:good.c}. Let $\epsilon \in [-1,1]$. Then for any $x \in [-1,1]$ with $\ex e^{- \gamma x} \le e^{\gamma \epsilon}$ we have
\[
\ex e^{c \eta^2 x^2 - \eta x}
~\le~
e^{c \eta^2 \epsilon^2 + \eta \epsilon}
.
\]
\end{theorem}

\begin{proof}
Taking expectation over Lemma~\ref{lem:pointwise}, we find
\[
\ex
e^{c \eta^2 x^2 - \eta x}
~\le~
e^{c \eta^2 \epsilon^2 + \eta \epsilon}
+
\frac{  \ex e^{- \gamma (x + \epsilon)}-1}{\gamma} \eta (1 + 2 c \eta \epsilon) e^{c \eta^2 \epsilon^2 + \eta \epsilon}
,
\]
and the claim follows by bounding the right-most term by $0$. (Note that the factor $1 + 2 c \eta \epsilon$ is positive by Lemma~\ref{lem:derv.sign}.)
\end{proof}

\section{Bernstein to Central}\label{appx:b2c}
An indexed family of random variables $\setc{x^f}{f \in \mathcal F}$ satisfies the $(B, \kappa)$-\emph{Bernstein condition} if
\[
\ex \sbr*{(x^f)^2} ~\le~ B \ex \sbr*{x^f}^\kappa \qquad \text{for all $f \in \mathcal F$}
\]
and it satisfies the $\epsilon$-\emph{Central condition} if
\[
- x^f ~\stochleq_\eta~ \epsilon(\eta) \qquad \text{for all $f \in \mathcal F$ and $\eta \ge 0$}
\]
We now show that Bernstein implies Central. This is a special case of \cite[Theorem~5.4, Part~1]{EGMRW15}.
Assume Bernstein. Then by the Bernstein Sandwich \cite[Lemma~C.2]{llr}, simplifying $e^\eta-\eta-1 \le \eta^2$, which holds for small enough $\eta \le 1.79328$, and by the Bernstein assumption
\[
\epsilon^f(\eta)
~=~
\frac{1}{\eta} \ln \ex \sbr*{e^{-\eta x^f}}
~\le~
\eta \ex[(x^f)^2] - \ex[x^f]
~\le~
\eta B \ex[x^f]^\kappa - \ex[x^f]
\]
Then (the maximizer is found at $x= (B \eta  \kappa)^{\frac{1}{1-\kappa}}$ (which is $\le 1$ when $B \kappa \eta \le 1$, so for small enough $\eta$ this is a reasonable point))
\[
\epsilon(\eta)
~=~
\sup_f \epsilon^f(\eta)
~\le~
\sup_x \eta B x^\kappa - x
~=~
\frac{1-\kappa}{\kappa} (B \eta  \kappa)^{\frac{1}{1-\kappa}}
.
\]

\section{Proof of Theorem~\ref{thm:luckiness.regret.bd}}\label{sec:pf9}
\begin{proof}
For the Hedge setting, let us write $x_t \df \ex_{k \sim w_t} [x_t^k]$ for the excess loss of the learner in round $t$. Then $x_t \in [-1,1]$ and $-x_t \stochleq_\eta \epsilon(\eta)$ by Lemma~\ref{lem:mix}. Now by definition of $x_t^k$ in \eqref{eq:excess.loss} and $R_T^k$ and $V_T^k$ (see Section~\ref{sec:hedge})
\begin{align*}
\sum_{t=1}^T x_t 
&~=~ \sum_{t=1}^T \del*{\tuple{w_t, \loss_t} - \loss_t^{k^*}} ~=~ R_T^{k^*}
\quad\text{and}
\\
\sum_{t=1}^T (x_t)^2 
&~=~ \sum_{t=1}^T \del*{\tuple{w_t, \loss_t} - \loss_t^{k^*}}^2 ~=~ V_T^{k^*}
.
\end{align*}
For OCO, let us write $x_t \df x_t^{w_t}$ for the excess loss of the learner in round $t$. Again $x_t \in [-1,1]$ and we have $- x_t \stochleq_\eta \epsilon(\eta)$ by construction of $\epsilon(\eta)$. Moreover, from the definition of $\tilde R_T^u$ and $V_T^u$ from Section~\ref{sec:oco},
\begin{align*}
\sum_{t=1}^T x_t
&~=~
\sum_{t=1}^T \tuple{w_t - u^*, \nabla \loss_t(w_t)}
~=~
\tilde R_T^{u^*}
\quad
\text{and}
\\
\sum_{t=1}^T (x_t)^2
&~=~
\sum_{t=1}^T \tuple{w_t - u^*, \nabla \loss_t(w_t)}^2
~=~
V_T^{u^*}
.
\end{align*}
With this notation the second-order regret bounds \eqref{eq:squint.bd} and \eqref{eq:metagrad.bd} both state
\begin{equation}\label{eq:abstract.bd}
\sum_{t=1}^T x_t ~\le~ \sqrt{ \del*{\sum_{t=1}^T x_t^2} K_T} + K_T.
\end{equation}
Now fix $\gamma \ge 0$. For any $t$, as $- x_t \stochleq_\eta \epsilon(\eta)$, Lemma~\ref{lem:squeezer} gives
\[
c \gamma x_t^2 - x_t
~\stochleq_\gamma~
\epsilon(2 \gamma) (1 + c \gamma^2)
\]
By telescoping over rounds (using the chain rule Lemma~\ref{lem:chainrule}), we obtain
\begin{equation}\label{eq:midpoint}
c \gamma \sum_{t=1}^T x_t^2 - \sum_{t=1}^T x_t
~\stochleq_\gamma~
T \epsilon(2 \gamma) (1 + c \gamma^2)
.
\end{equation}
The individual sequence regret bound \eqref{eq:abstract.bd} gives us (since $2 \sqrt{a b} = \inf_\eta \eta a + b/\eta$) for every $\eta \ge 0$
\[
\sum_{t=1}^T x_t
~\le~
\frac{\eta}{2} \sum_{t=1}^T x_t^2
+ \frac{
  K_T
}{2 \eta}
+ K_T
.
\]
Plugging in $\eta = c \gamma$ (this implies $\eta \in [0,1/2]$ and $\gamma = \frac{2 \eta}{1 - 4 \eta^2}$) and combination with \eqref{eq:midpoint} results in
\[
\sum_{t=1}^T x_t
~\stochleq_\gamma~
\frac{K_T}{c \gamma}
+ T \epsilon(2 \gamma) (1 + c \gamma^2)
+ 2 K_T
.
\]
For the Hedge setting this proves the theorem. For OCO we finish with $R_T^{u^*} \le \tilde R_T^{u^*}$.
\end{proof}

\section{Proof of Lemma~\ref{lem:hingeloss}}\label{sec:pf5}
\begin{proof}
  Since, by assumption, $\u$ and $\X$ have length at most $1$, the hinge
  loss simplifies to $\loss(\u) = 1 - Y \ip{\u}{\X}$ with gradient $\nabla
  \loss(\u) = -Y \X$. This implies that
  \begin{equation}
    \u^* \df \argmin_\u \ex\sbr*{\loss(\u)} = \frac{\vmu}{\|\vmu\|},
  \end{equation}
  and
  \begin{align*}
    (\w - \u^*)^\top \ex &\sbr*{\nabla \loss(\w) \nabla \loss(\w)^\top}(\w - \u^*)
      = (\w - \u^*)^\top \ex \sbr*{\X \X^\top}(\w - \u^*)\\
      &\leq \lambdamax (\w - \u^*)^\top (\w - \u^*)
      \leq 2\lambdamax (1 - \ip{\w}{\u^*})\\
      &= \frac{2\lambdamax}{\|\vmu\|} (\w -\u^*)^\top (-\vmu)
      = \frac{2\lambdamax}{\|\vmu\|} (\w -\u^*)^\top \ex\sbr*{\nabla
      \loss(\w)},
  \end{align*}
  which proves the first part of the lemma

  For the second part, we first observe that $\lambdamax = 1/d$. Then,
  to compute $\|\vmu\|$, assume without loss of generality that
  $\|\bar{\u}\| = 1$, in which case $\bar{\u} = \u^*$. Now symmetry of
  the distribution of $\X$ conditional on $\ip{\X}{\u^*}$ gives
  \begin{multline*}
    \ex\sbr*{Y \X \mid \ip{\X}{\u^*}}
    \\
      = \sign(\ip{\X}{\u^*}) \ex\sbr*{\X \mid \ip{\X}{\u^*}}
      = \sign(\ip{\X}{\u^*}) \ip{\X}{\u^*} \u^*
      = |\ip{\X}{\u^*}| \u^*.
    \end{multline*}  
  By rotational symmetry, we may further assume without loss of
  generality that $\u^* = \e_1$ is the first unit vector in the standard
  basis, and therefore
  \begin{equation*}
   \|\vmu\| = \|\ex\sbr*{|\ip{\X}{\u^*}|} \u^*\| =  \ex\sbr*{|X_1|}.
  \end{equation*}
  If $\Z = (Z_1,\ldots,Z_d)$ is multivariate Gaussian $\normal(0,I)$.
  Then $\X = \Z/\|\Z\|$ is uniformly distributed on the sphere, so
  \begin{align*}
    \ex\sbr{|X_1|}
      = \ex\sbr*{\frac{|Z_1|}{\|\Z\|}}
      \geq \frac{1}{4 \sqrt{d}} \Pr\del*{|Z_1| \geq \half \wedge \|\Z\|
      \leq 2\sqrt{d}}.
  \end{align*}
  Since $\Pr\del*{|Z_1| < \half} \leq 0.4$ and $\Pr\del*{\|\Z\| \geq
  2\sqrt{d}} \leq \frac{1}{4d} \ex\sbr*{\|\Z\|^2} = \frac{1}{4}$, we
  have
  \begin{equation*}
    \Pr\del*{|Z_1| \geq \half \wedge \|\Z\|
        \leq 2\sqrt{d}}
    \geq 1 - 0.4 - \frac{1}{4} = 0.35,
  \end{equation*}
  from which the conclusion of the second part follows with $c =
  8/0.35$.
\end{proof}

\section{Proof of Theorem~\ref{thm:main}}\label{sec:pf3}

\begin{proof}
As pointed out below Condition~\ref{con:Cen}, the $(B, \kappa)$-Bernstein condition implies the central condition with $\epsilon(\eta) \le (\eta B)^{\frac{1}{1-\kappa}}$. 
By Theorem~\ref{thm:luckiness.regret.bd}, using $1/c \le 2 (1+ \gamma^2)$ and $c \le \frac{1}{2}$, we find that for all $\gamma \ge 0$
\begin{equation}\label{eq:stprt}
R_T^{f^*}
~\stochleq_\gamma~
(1+ \gamma^2) \frac{2 K_T}{\gamma}
+ (1 + \frac{1}{2} \gamma^2) \epsilon(2 \gamma)  T
+ 2 K_T
.
\end{equation}
By Lemma~\ref{lem:ESI} this implies for all $\gamma \ge 0$
\[
\ex [R_T^{f^*}]
~\le~
(1+ \gamma^2) \frac{2 K_T}{\gamma}
+ (1 + \frac{1}{2} \gamma^2) \epsilon(2 \gamma)  T
+ 2 K_T
.
\]
It remains to tune $\gamma$ to exploit the stochastic condition expressed by $\epsilon$. 
Reducing the above right-hand-side expression to its main terms and setting the derivative to zero suggests picking
\[
\hat \gamma 
~=~ 
\del*{\frac{2 K_T (1-\kappa) (2 B)^{-\frac{1}{1-\kappa}}}{T}}^{\frac{1-\kappa}{2-\kappa}}
~=~
O\del*{
  \del*{\wfrac{K_T}{T}}^{\frac{1-\kappa}{2-\kappa}}
}
.
\]
Plugging in this tuning, we find
\[
\ex[R_T^{f^*}]
~\le~
(2-\kappa)
\del*{4 K_T B}^{\frac{1}{2-\kappa}}
(T/{(1-\kappa)})^{\frac{1-\kappa}{2-\kappa}}
+ (5-\kappa) K_T
\]
Finally, using that $
(2-\kappa)
\del*{4 B}^{\frac{1}{2-\kappa}}
(1/{(1-\kappa)})^{\frac{1-\kappa}{2-\kappa}}$ is maximized in $\kappa$ at $\kappa = 1-\frac{1}{4 B}$ where it takes value $1+4 B$, we may simplify this to
\[
\ex [R_T^{f^*}]
~\le~
\del*{1 + 4 B}
(K_T/4)^{\frac{1}{2-\kappa}} T^{\frac{1-\kappa}{2-\kappa}}
+ (5-\kappa) K_T
~=~
O\del*{
  K_T^{\frac{1}{2-\kappa}} T^{\frac{1-\kappa}{2-\kappa}}
}
,
\]
which gives the first claim of the Theorem. Lemma~\ref{lem:ESI} applied to \eqref{eq:stprt} also implies that for all $\delta \ge 0$ with probability at least $1-\delta$
\[
R_T^{f^*}
~\le~
(1+ \gamma^2) \frac{2 K_T}{\gamma}
+ (1 + \frac{1}{2} \gamma^2) \epsilon(2 \gamma)  T
+ 2 K_T
+ \frac{-\ln \delta}{\gamma}
.
\]
Tuning $\gamma$ as before with $K_T$ replaced by $K_T + \frac{- \ln \delta}{2}$ yields the second claim.
\end{proof}

\section{Continuous Models}\label{app:cont}

We now consider Squint with models of predictors ${\cal F}$ that have
uncountably many elements so that in general $\pi^{f^*}= 0$, and each
$f \in {\cal F}$ is a function from ${\cal X}$ to ${\cal A}$, with
$\ell^f_t := \ell(y_t,f(x_t))$ for some fixed loss function
$\ell: {\cal Y} \times {\cal A} \rightarrow [0,1]$. This setting includes standard parametric models in classification and regression
but also countable unions thereof as well as nonparametric models. We
first present an extension of Theorem~\ref{thm:main} to this case; we
then give an illustration of this result with sup-norm metric entropy
numbers.

Squint can be straightforwardly applied to uncountable models, but now
the weight vector $w_t$ output by Squint at time $t$ takes the form of
a distribution on the set ${\cal F}$. For general distributions $u$ on
${\cal F}$, the loss $u$ incurs at time $t$ is now defined as
$\ell^u_t := \ex_{f \sim u} [\ell_t^f]$, so that the loss of Squint at
time $t$ is given by $\ell_t^{w_t}$, which generalizes the expression
$\langle w_t,\ell_t\rangle$ for the countable case. The regret of
Squint relative to an arbitrary $u$ is thus given by
$R_T^u = \ex_{f \sim u} \sum_{t=1}^T (\ell^{w_t}_t - \ell^f_t)$, and
the variance term in (\ref{eq:squint.bd}) generalizes to
$V_T^u = \ex_{f \sim u} \sum_{t=1}^T v^f_t$ with
$v^f_t = (\ell_t^{w_t}- \ell_t^f)^2$.

For such models we will use that, as shown by \citet{koolen2015blog},
Squint satisfies the following quantile or `KL' bound:
\begin{equation}\label{eq:assumed.bound}
R_T^u ~\le~ 2 \sqrt{V_T^u K^u_T } + K^u_T
\end{equation}
which holds for every distribution $u$ on ${\cal F}$ and prior $\pi$, 
where now $K^u_T = O(\KL \delcc{u}{\pi} + \ln \ln T)$ and $\KL \delcc{u}{\pi}$ is the KL divergence between prior $\pi$ and the distribution $u$. 

Note that (\ref{eq:assumed.bound})
generalizes the countable bound (\ref{eq:squint.bd}), which is
retrieved if $u$ is taken to be a point mass on $k$.

\begin{theorem}[Extension of Theorem~\ref{thm:main}]\label{thm:mainb}
  In any stochastic setting satisfying the $(B,\kappa$)-Bernstein
  Condition~\ref{cond:Bern}, the guarantee \eqref{eq:assumed.bound}
  for Squint implies fast rates for Squint in expectation (if there is sufficient prior mass on $f$ that behave similarly to  $f^*$ 
in expectation) and
  with high probability (if there is sufficient prior mass on $f$ taht are guaranteed to behave similarly to $f^*$ on all $x$). That is, for all $T$, for any sequence $u_1, u_2, \ldots$ of distributions on ${\cal F}$ and sequence of  numbers $C_1, C_2, \ldots$ that satisfy 
\begin{equation}\label{eq:hurry}
\ex \left[\sum_{t=1}^T \ell_t^{u_T} \right] \leq  
\ex \left[\sum_{t=1}^T \ell_t^{f^*}\right] + C_T,
\end{equation}
we have 
\begin{align*}
\ex[R_T^{f^*}] 
&~=~ 
O\del*{(K_T+C_T)^{\frac{1}{2-\kappa}}T^{\frac{1-\kappa}{2-\kappa}}}
,
  \intertext{and if (\ref{eq:hurry}) holds for {\em every\/} sequence $(x^T,y^T)$, then we also  have for any $\delta > 0$, with probability at least $1-\delta$,}
R_T^{f^*} 
&~=~
O\del*{(K_T + C_T - \ln \delta)^{\frac{1}{2-\kappa}}T^{\frac{1-\kappa}{2-\kappa}}}
\end{align*}
where  $K_T \df K_T^{u_T}$ from~\eqref{eq:assumed.bound}.
\end{theorem}
While this theorem does allow us to use priors $u$ with uncountable
support, it is easiest to illustrate with priors with support on a
discretized version (countable subset) of ${\cal F}$ which may assign probability $0$ to $f^*$:
\newcommand{\model}{{\ensuremath{\cal F}}}
\newcommand{\cF}{{\ensuremath{\cal F}}}
\begin{example}{\rm 
Consider the classification setting where ${\cal J}$ is either finite or ${\mathbb N}$, and 
$\model = \bigcup_{j \in {\cal J}} \model_j$ is a finite or countable union of 
sub-models such that for $\delta > 0$,
$\ddot{{\cal F}}_{j,\delta} \subset \cF_j$ is a minimal $\delta$-cover of
$\ddot{{\cal F}}_j$ in the $\ell_{\infty}$-norm 
(that is, we require $\sup_{f \in \cF_j} \min_{\dot{f} \in \ddot{\cF}_{j,\delta}} \sup_{x \in {\cal X}, y \in {\cal Y}} 
\| \ell(y,f(x) - \ell(y,\dot{f}(x) \| \leq \delta$). 
Define $\Gamma := \{2^0, 2^{-1}, \ldots \}$. 
Assume that
for all $j$, $\mathcal{N}(\cF_j,\delta) := |\ddot{\cF}_{j,\delta}|  < \infty$ 
and note that $\log \mathcal{N}(\cF_j,\delta)$ is the metric entropy of $\cF_j$ in the sup norm at scale $\delta$. Let $\pi_{{\cal J}}$ be a probability mass function on ${\cal J }$ and let $\pi_{{\mathbb N}}$ be a probability distribution on ${\mathbb N}$ with $-\log \pi_{{\cal J}}(j)\pi_{{\mathbb N}}(k) = O (\log (j k)) $ 
 and let $\pi$ be the prior on
$\bigcup_{j \in {\mathbf N}, \delta \in \Gamma}
\ddot{F}_{j,\delta}$ with mass function $\pi$ given by, for
$f \in \ddot{F}_{j,2^{-k}}$,
$\pi(f) = \pi_{{\cal J}}(j)  \pi_{\mathbb N}(j) / \mathcal{N}(\cF_j,2^{-k}))$. Then 
Theorem~\ref{thm:mainb} gives the following bound in expectation (and mutatis m utandis in probability): 
$$
R_T^{f^*} 
~=~
O\left( 
\left( T 2^{-k} + \min_{j,k} \log (j k)  + \log \mathcal{N}(\cF_j,2^{-k})  \right)^{\frac{1}{2-\kappa}}T^{\frac{1-\kappa}{2-\kappa}} \right).
$$

Bounds in terms of models with bounded $\ell_{\infty}$-entropy numbers
were considered before by, e.g. \cite{gaillard2015chaining} with
bounded squared error loss. We note that, if ${\cal F}$ has
logarithmic entropy numbers (e.g. ${\cal F} = {\cal F}_1$ and
$\log \mathcal{N}(\cF_1,\epsilon) = O(- \log \epsilon)$, then, by
plugging in $k = \lceil \log_2 T \rceil$, we find that this cumulative
regret bound is of the form
$O ((\log T) \cdot T^{\frac{1- \kappa}{2- \kappa}})$, the standard
rate referred to in the discussion underneath
Theorem~\ref{thm:main}. In the case of larger (polynomial) entropy
numbers, our bounds are presumably suboptimal compared to the bounds
that can be obtained by ERM, since Squint is essentially a form of an
exponentially weighted forecaster that cannot exploit the chaining
technique, viz. the discussion by \citet{gaillard2015chaining},
\citet{audibert2009fast} and \citet{rakhlin2014online}. Nevertheless,
unlike ERM, Squint is robust and will continue to achieve nontrivial regret under
nonstochastic, adversarially generated data, even with polynomial
entropy numbers.

In practice, one may often work with ${\cal F}_i$ which have small (e.g. logarithmic) entropy numbers relative to the   pseudo-distance $d(f_1,f_2) = {\mathbb P}(f_1(X) \neq f_2(X))$ considered by e.g. \cite{Tsybakov04,Audibert04}, which may be much smaller than the $\ell_{\infty}$-numbers. In such cases, Theorem~\ref{thm:mainb} can still be used to give good bounds in expectation.
}\end{example}

\paragraph{Proof of Theorem~\ref{thm:mainb}}
Consider a (for now) arbitrary sequence $u_1, u_2, \ldots$, define
$K_T := K_T^{u_T}$ and $K'_T = K_T/4$, and
$$
C'_T = - (R_T^{u_T} -  R_T^{f^*})
\text{\ or equivalently\ } \sum_{t=1}^T \ell_t^{u_T} =  
\sum_{t=1}^T \ell_t^{f^*} + C'_T.
$$
One easily shows that for general $a,b,c \in {\mathbb R}$, one has
$(a-b)^2/2 \leq (a -c)^2 + (b-c)^2$. Applying the statement with $a
= \ell^{w_t}_t$, $b = \ell^{f}_t$ and $c= \ell^{f^*_t}$ gives $v^f_t
\leq 2 (v_t^{f^*} + (x_t^f)^2)$.  Summing over $t = 1..T$ and taking
expectation over $f \sim u_T$ now gives $V_T^{u_T} \leq 2 V_T^{f^*} +
2 E_T$ where $E_T = \ex_{f \sim u_T} \left[ \sum_{t=1}^T (x_t^f)^2\right]$.

Applying this to the bound \eqref{eq:assumed.bound} above at $u_T$, we get 
\begin{equation}\label{eq:hurryb}
R_T^{f^*}
~\le~
C'_T + 2 \sqrt{(V_T^{f^*}+ E_T) 2 K'_T} +K'_T 
~=~
\inf_\eta \set*{
  C'_T + \eta (V_T^{f^*}+E_T) + \frac{2 K'_T}{\eta} + K'_T
}.
\end{equation}
We first prove an analogue to Theorem~\ref{thm:luckiness.regret.bd}
for the uncountable setting, based on (\ref{eq:hurryb}).  As in that theorem, let, for given
${\cal F}$, $\setc*{x_t^f}{f \in \mathcal F}$ be the associated the
excess loss family from \eqref{eq:excess.loss}, and let
$\epsilon(\eta)$ be, as in \eqref{eq:uniform.eta}, the corresponding
maximal normalized cumulant generating function. Let $K'_T$ be as
above.  Fix $\gamma \ge 0$ and let $c$ be as in
Lemma~\ref{lem:squeezer}.  Now as in the proof of Theorem~\ref{thm:luckiness.regret.bd} we have for all $f \in {\cal F}$,  $x^f_t \in [-1,1]$ and  $- x^f_t \stochleq_\eta \epsilon(\eta)$ by construction of $\epsilon(\eta)$. Hence
 $- x_t^f \stochleq_{\gamma} \epsilon(\gamma)$ for all $f \in \mathcal F$, which  implies $- \ex_{f \sim w_t} x_t^f \stochleq_{\gamma} \epsilon(\gamma)$
and also $- \ex_{f \sim u_T} x_t^f \stochleq_{\gamma} \epsilon(\gamma)$
, and hence by Lemma~\ref{lem:squeezer} (see remark below the lemma), 
\begin{align} \label{eq:first}
c \gamma \ex_{f \sim w_t}[x_t^f]^2 -  \ex_{f\sim w_t}[x_t^f]
& \stochleq_{\gamma}  \epsilon(2 \gamma) (1 + c \gamma^2) \text{\ and\ } \\
\label{eq:second} c \gamma \ex_{f \sim u_T}[x_t^f]^2 - \ex_{f\sim u_T}[x_t^f]
& \stochleq_{\gamma} \epsilon(2 \gamma) (1 + c \gamma^2).
\end{align}
Using  $r_t^{f^*} = \ex_{f \sim w_t} \sbr*{x_t^f}$, again analogously to the proof of Theorem~\ref{thm:luckiness.regret.bd}, we may telescope (\ref{eq:first}) over rounds to
\begin{equation}\label{eq:secondenhalf}
c \gamma V_T^{f^*} - R_T^{f^*}
~\stochleq_{\gamma}
T \epsilon(2 \gamma) (1 + c \gamma^2)
\end{equation}
Now we use (\ref{eq:hurryb}) with  $\eta = \frac{c \gamma}{2}$, which implies 
$2 R_T^{f^*} \leq 2  C'_T + c \gamma (V^{f^*}_T + E_T) + 4 K'_T/ (c \gamma) + 2 K'_T$. Combining this with (\ref{eq:secondenhalf}), we find: 
\begin{equation}\label{eq:third}
U ~\stochleq~0 \text{\ with\ } U = \gamma R_T^{f^*}
- \left(
2 \gamma C'_T
+ c \gamma^2 E_T
+ \frac{4 K'_T}{c} + \gamma 2 K'_T
+ \gamma T \epsilon(2 \gamma) (1 + c \gamma^2) \right).
\end{equation}
Similarly to deriving (\ref{eq:secondenhalf}), using the definition of
$E_T$, we may telescope (\ref{eq:second}) over rounds to get
\begin{equation}\label{eq:fourth}
U'~\stochleq~0 \text{\ with \ } U' = c \gamma^2 E_T - \gamma C'_T
-
\gamma T \epsilon(2 \gamma) (1 + c \gamma^2). 
\end{equation}
We may now combine (\ref{eq:third}) and (\ref{eq:fourth}) using Lemma~\ref{lem:mix} with $w$ a distribution that puts mass $1/2$ on 
random variable $U$ and $1/2$ on $U'$, to get $(U + U')/2 \stochleq 0$, which can be rewritten to: 
\begin{multline} \nonumber
\frac{\gamma}{2} \Bigg( 
R_T^{f^*}
- (
2  C'_T
+ c \gamma^ E_T
+ \frac{4 K'_T}{c \gamma } + 2 K'_T 
+ T \epsilon(2 \gamma) (1 + c \gamma^2)) 
\\
+ c \gamma E_T - \gamma C'_T
-
\gamma  \epsilon(2 \gamma) (1 + c \gamma^2)
\Bigg)  ~\stochleq~0,
\end{multline}
and further to 
\begin{equation}\label{eq:hurryc}
\frac{1}{2}  R_T^{f^*} 
~\stochleq_{\gamma}~  \frac{3 }{2} C'_T + \frac{K_T}{c \gamma } + T \epsilon(2 \gamma) (1 + c \gamma^2) + 2 K'_T, 
\end{equation}
which is the required analogue of the statement of
Theorem~\ref{thm:luckiness.regret.bd}. Note that this statement holds
for {\em every\/} sequence $u_1, \ldots, u_T$, and $C'_T$ is a random
variable that depends on data $(x^T,y^T)$. 

The remainder of the proof of Theorem~\ref{thm:mainb} now follows in a fashion entirely analogous to the proof of Theorem~\ref{thm:main}, as in Appendix~\ref{sec:pf3}, where we use that we can bound $C'_T$ by $C_T$, either in expectation or on all sequences; we omit further details where one uses (\ref{eq:hurryc}) instead of the corresponding statement of Theorem~\ref{thm:luckiness.regret.bd}; we omit further details. 
\end{document}